\documentclass[11pt]{article}
\usepackage[preprint]{neurips_2026}
\usepackage[utf8]{inputenc}
\usepackage[T1]{fontenc}
\usepackage{amsmath,amssymb,amsfonts,amsthm}
\usepackage{bm}
\usepackage{booktabs}
\usepackage{algorithm}
\usepackage{algorithmic}
\usepackage{graphicx}
\usepackage{hyperref}
\usepackage{url}
\usepackage{xcolor}

\newcommand{\R}{\mathbb{R}}
\newcommand{\E}{\mathbb{E}}
\newcommand{\HSIC}{\mathrm{HSIC}}
\newcommand{\dHSIC}{d\mathrm{HSIC}}
\newcommand{\indep}{\mathrel{\perp\!\!\!\perp}}
\newcommand{\xHSIC}{x\mathrm{HSIC}}
\newcommand{\mMMD}{m\mathrm{MMD}}
\newcommand{\mHSIC}{m\mathrm{HSIC}}
\newcommand{\mdHSIC}{md\mathrm{HSIC}}
\newcommand{\xdHSIC}{xd\mathrm{HSIC}}
\newcommand{\one}{\mathbf{1}}

\DeclareMathOperator{\tr}{tr}
\newtheorem{theorem}{Theorem}

\newtheorem{lemma}[theorem]{Lemma}

\newtheorem{remark}[theorem]{Remark}
\newtheorem{assumption}[theorem]{Assumption}

\title{A Martingale Kernel Independence Test}

\author{%
  Felix Laumann \\
  Imperial College London \\
  \texttt{f.laumann18@imperial.ac.uk} \\
  \And
  Zhaolu Liu \\
  Imperial College London \\
  \And
  Mauricio Barahona \\
  Imperial College London \\
}

\begin{document}

\maketitle

\begin{abstract}
The Hilbert--Schmidt Independence Criterion (HSIC) and its
joint-independence extension $\dHSIC$ are degenerate $V$-statistics
whose data-dependent weighted-$\chi^2$ null limits force a
permutation calibration that multiplies the per-test cost by the
number of permutations, in practice two orders of magnitude.
Adapting the recent martingale MMD construction for two-sample
testing to the (joint) independence problem, we introduce two
studentised statistics whose null distributions are standard normal
regardless of the data law, so that a single normal-quantile lookup
replaces the permutation step entirely.  The first, $\mHSIC$, is a
self-normalised lower-triangular sum of the Hadamard product of two
empirically centred Gram matrices. Under independence and
bounded-fourth-moment kernels it converges to a standard normal. It is
consistent against every fixed alternative, and runs at quadratic
cost in the sample size without any sample split, matching the
biased HSIC $V$-statistic.  The naive direct-product extension to
$d$ variables fails because the empirical-centring bias compounds
across the $d$ factors and breaks finite-sample calibration once the
number of variables exceeds the square root of the sample size.
Our second statistic, $\mdHSIC$, repairs this with a single
half-sample split: the centring is estimated on one half and the
lower-triangular self-normalised martingale is run on the other,
shrinking the conditional-mean residual to a quantity that is
exponentially small in $d$, so the statistic is asymptotically
standard normal at every fixed number of jointly tested variables,
with a per-test cost that grows only linearly in $d$.  On synthetic
data with per-variable input dimension from $1$ to $500$ and between
$2$ and $10$ jointly tested variables, both statistics match the
empirical type-I error rate and test power of permutation-calibrated
baselines while running $25$ to $60\times$ faster.
\end{abstract}

\section{Introduction}
\label{sec:intro}

The Hilbert--Schmidt Independence Criterion (HSIC)
\citep{gretton2008} embeds the joint distribution $P_{XY}$ and the
product distribution $P_X \otimes P_Y$ of two random variables
$X, Y$ into a reproducing kernel Hilbert space (RKHS) and
returns the squared distance.  For characteristic kernels, this distance
is zero iff $X \indep Y$, so HSIC is a fully non-parametric measure of
dependence.  Given a paired sample $\{(X_i, Y_i)\}_{i=1}^{n}$ of size
$n$ and bounded characteristic kernels $k_X, k_Y$ on the marginal
sample spaces, write $K_X, K_Y \in \R^{n\times n}$ for the
corresponding Gram matrices $(K_X)_{ij} := k_X(X_i, X_j)$,
$(K_Y)_{ij} := k_Y(Y_i, Y_j)$.  The standard biased $V$-statistic for
HSIC is then
\begin{equation}
\widehat\HSIC_n \;=\; \tfrac{1}{n^2}\,\tr(K_X H K_Y H), \qquad
H \;:=\; I_n - \tfrac{1}{n}\one\one^\top,
\label{eq:hsic-emp}
\end{equation}
where $I_n$ is the $n\times n$ identity matrix and $\one\in\R^n$ is
the vector of all ones, so that $H$ is the empirical centring matrix.
This estimator is degenerate under the null hypothesis
$H_0:X\indep Y$, i.e.\ the rescaled statistic $n\,\widehat\HSIC_n$
converges to a weighted sum of $\chi^2$ random variables whose
weights depend on the unknown joint law
\citep[Theorem 4.1]{gretton2008}. 
Practical implementations therefore
approximate the distribution under the null hypothesis by a
permutation test, which requires $B$ recomputations of the statistic
and inflates the per-test cost from $O(n^2)$ to $O(B\,n^2)$.

The same degeneracy is known to arise in the two-sample problem, where
the unbiased estimator of the squared maximum mean discrepancy (MMD) between two distributions
$P, Q$ has an analogous infinite weighted-$\chi^2$ null limit
\citep{gretton2012kernel}.  
\citet{chatterjee2025} recently resolved the
calibration step in that setting by observing that a slight modification
of the standard quadratic-time estimator
\begin{align}
T_n &\;=\; \frac{1}{n}\sum_{i=2}^{n}\frac{1}{i}\sum_{j=1}^{i-1}
\Phi^{\mMMD}(Z_i, Z_j),\notag\\
\Phi^{\mMMD}(Z_i, Z_j) &\;:=\; k(X_i, X_j) - k(X_i, Y_j) - k(X_j, Y_i) + k(Y_i, Y_j),
\label{eq:mmmd}
\end{align}
with $Z_i := (X_i, Y_i)$, $\{X_i\}_{i=1}^n\sim P$,
$\{Y_i\}_{i=1}^n\sim Q$ and a single bounded characteristic kernel $k$
on the common sample space, admits a martingale-difference structure
under $H_0:P=Q$.  Throughout, write
$\mathcal F_{i-1} := (Z_1, \ldots, Z_{i-1})$ for the past observations
and $\E[\,\cdot\,|\,\mathcal F_{i-1}]$ for the corresponding
conditional expectation.  Each inner
sum $\frac1i\sum_{j<i}\Phi^{\mMMD}(Z_i, Z_j)$ is the RKHS inner
product of an empirical witness function determined by
$\mathcal F_{i-1}$, $\hat f_i := \tfrac{1}{i}\sum_{j<i}\bigl[k(X_j,\cdot) -
k(Y_j,\cdot)\bigr]$, with the increment $k(X_i,\cdot) - k(Y_i,\cdot)$
at the current sample, which has zero conditional mean under $H_0$.
Studentising $T_n$ by its own empirical quadratic variation
\[
  \sigma_n^2 \;:=\; \frac{1}{n^2}\sum_{i=2}^{n}
  \Bigl(\tfrac{1}{i}\sum_{j<i} \Phi^{\mMMD}(Z_i, Z_j)\Bigr)^{\!2}
\]
gives the self-normalised statistic
$\eta_n := T_n/\sigma_n$. 
The self-normalised martingale Central Limit Theorem (CLT) \citep{fan2018berry} then yields
$\eta_n \xrightarrow{d} \mathcal N(0, 1)$ under $H_0$
without any dependence on $P$, so a single standard-normal quantile replaces the permutation step.  
The construction recovers the
$O(n^2)$ cost of the standard estimator and the test is consistent
against every fixed $P\neq Q$.

The paper carries out the extension of this paradigm to kernel independence testing.  
Concretely, we
\begin{enumerate}
  \item Define the \textbf{martingale HSIC} ($\mHSIC$) statistic, a
    studentised lower-triangular sum of the Hadamard product of two
    empirically centred Gram matrices.  Under $H_0:X\indep Y$ and the
    moment conditions of \citet[Theorem 3.1]{chatterjee2025}, $\mHSIC$
    converges     to $\mathcal N(0, 1)$ (the standard normal distribution)
    regardless of $P_{XY}$ (Theorem~\ref{thm:mhsic-null}); under any
    fixed alternative for which $\HSIC(P_{XY}) > 0$ the test is consistent
    by a calculation parallel to \citet[Theorem 5.1]{chatterjee2025},
    at a per-test cost of $O(n^2)$, identical to the biased HSIC
    $V$-statistic but without the $O(B)$ permutation resampling step.
  \item Define the \textbf{split-martingale $d$-variable HSIC}
    ($\mdHSIC$) statistic, the extension to mutual independence among
    $d$ random variables.  The naive direct-product analogue of
    $\mHSIC$ has an $O_P(d/\sqrt n)$ empirical-centring bias that
    dominates the studentised scale at $d \gtrsim \sqrt n$ and breaks
    finite-sample calibration. 
    We correct this behaviour with a single half-sample
    split (as in $\xdHSIC$ of \citet{liu2025}) that replaces the
    full-sample centring $\mu_k$ by the first-half empirical mean
    embedding $\hat\mu_k^{\mathcal S_1}$, and runs the
    lower-triangular self-normalised martingale on the second half.
    Under joint independence and the analogous moment condition,
    $\mdHSIC$ converges to $\mathcal N(0, 1)$ at every fixed
    $d \geq 2$ (Theorem~\ref{thm:mdhsic-null}), at a per-test cost
    of $O(d\,n^2)$, linear in $d$.
  \item Empirically validate both statistics on synthetic data-generating processes (DGPs).
    For $\mHSIC$ we assess various per-variable input dimension of $X$
    and $Y$, which we call the \emph{ambient} dimension
    $d_{\mathrm{ambient}}$, 
    to distinguish it from the integer $d$ that elsewhere in the paper counts random variables. 
    For $\mdHSIC$ we iterate over $d$ jointly tested variables.  
    We report runtime against
    the permutation-calibrated HSIC and $\dHSIC$ baselines
    (Section~\ref{sec:experiments}).
\end{enumerate}

The technical novelty lies in the choice of the bivariate kernel that
is substituted into the lower-triangular martingale construction.
Three properties of that kernel suffice for the studentised sum to
converge to a standard normal under the null: (i)~symmetry, so the
sum is unbiased for the population functional; (ii)~conditional
expectation given the past that vanishes at a rate fast enough to be
absorbed by the studentisation, making the sum a martingale-difference
sequence; (iii)~boundedness with finite fourth moment, so the
self-normalised martingale CLT of \citet{fan2018berry} applies.  For
the two-sample problem, the kernel $\Phi^{\mMMD}$ in~\eqref{eq:mmmd}
satisfies all three trivially.  For independence ($d = 2$), the
centred-kernel product $\bar k_X\,\bar k_Y$ satisfies~(ii) when the
centring uses the marginal mean embeddings, and the
empirical-to-population gap is $O_P(1/\sqrt n)$ and asymptotically
negligible after studentisation
(Lemma~\ref{lem:Phi-cond}).  For joint independence ($d \geq 3$) the
per-variable gap multiplies across the $d$ factors, producing an
$O_P(d/\sqrt n)$ conditional-mean residual that dominates the
studentised scale whenever $d \gtrsim \sqrt n$
(Section~\ref{sec:mdhsic-why-split}).  Our remedy is a single
half-sample split, in the spirit of $\xdHSIC$ \citep{liu2025}: the
centring is estimated from the first half and the lower-triangular
martingale is run on the second.  The conditional-mean residual then
collapses to $O_P(m^{-d/2})$, exponentially small in $d$ and
negligible for every $d \geq 2$
(Lemma~\ref{lem:Phi-d-split-cond}).

\section{Related Work}
\label{sec:related}

HSIC was introduced by \citet{gretton2008}, in which the standard biased $V$-statistic uses the
Hadamard product of two empirically centred Gram matrices, with
calibration under $H_0$ by permutation~\citep[\S 4]{gretton2008} or
a Gamma moment match.  $\dHSIC$ \citep{pfister2018} generalises
HSIC to joint independence among $d \geq 2$ variables via a $2d$-th
order $V$-statistic, again calibrated by permutation, bootstrap, or
Gamma; none of these is permutation-free.  The unconditional
permutation-free programme begins with the cross-statistic $\xHSIC$
of \citet{shekhar2023}, which splits the sample in two halves,
takes the inner product of the per-half cross-covariance estimates,
and studentises the result with a sample standard error;
\citet{liu2025} extend this construction to higher-order
interactions ($\mathrm{xdHSIC}, \mathrm{xLI}, \mathrm{xSI}$), still
by sample splitting.  An independent line of work avoids the split
and exploits martingale concentration:
\citet{balsubramani2016sequential} construct a sequential MMD test
based on linear-time martingale concentration inequalities,
\citet{shekhar2023nonparametric} give an any-time-valid kernel test
based on Ville's inequality, and \citet{chatterjee2025} construct
the $\mMMD$ statistic that we summarise above.  The latter is the
closest predecessor to the present paper. 
It exploits the
lower-triangular $V$-statistic of the standard MMD as a sum of
martingale differences and applies the self-normalised martingale
CLT of \citet{fan2018berry}.  We adapt that construction to the
(joint) independence problem. $\mHSIC$ is permutation-free
\emph{and} sample-split-free, while $\mdHSIC$ is permutation-free
and uses a single half-sample split for the centring (as in
$\xdHSIC$ of \citet{liu2025}) but then runs the lower-triangular
martingale on the second half rather than the rectangular
cross-block of $\xdHSIC$, which retains the $\sqrt 2$ factor in
studentised power relative to rectangular-block cross-statistics.

\section{Background}
\label{sec:background}

Let $\mathcal X$ be a Polish space and $k:\mathcal X\times\mathcal X\to\R$
a positive-definite, bounded, characteristic kernel with associated
RKHS $\mathcal H_k$ and feature map $\phi(x) = k(\cdot, x)$.  The mean
embedding of a probability measure $P\in\mathcal P(\mathcal X)$ is
$\mu_P = \E_{X\sim P}[\phi(X)]$ and the centred kernel is
\begin{equation}
\bar k(x,x') \;=\; \langle \phi(x) - \mu_P,\, \phi(x') - \mu_P\rangle_{\mathcal H_k}
   \;=\; k(x,x') - \mu_P(x) - \mu_P(x') + \langle\mu_P, \mu_P\rangle,
\label{eq:centred-kernel}
\end{equation}
with $\mu_P(x) := \E[k(X, x)]$.  Throughout, given a sample
$\{X_i\}_{i=1}^n$ of size $n$, we write $K\in\R^{n\times n}$ for the
Gram matrix with entries $K_{ij} := k(X_i, X_j)$, $I_n$ for the
$n\times n$ identity matrix, $\one\in\R^n$ for the vector of all ones,
and $H := I_n - \tfrac{1}{n}\one\one^\top$ for the corresponding
empirical centring matrix.  The empirically centred Gram matrix has
entries
$(HKH)_{ij} = k(X_i, X_j) - \tfrac{1}{n}\sum_a k(X_a, X_i) -
\tfrac{1}{n}\sum_a k(X_a, X_j) + \tfrac{1}{n^2}\sum_{a,b} k(X_a, X_b)$.

Given paired observations $\{(X_i, Y_i)\}_{i=1}^n$ with bounded
characteristic kernels $k_X, k_Y$, RKHSs
$\mathcal H_X := \mathcal H_{k_X}$,
$\mathcal H_Y := \mathcal H_{k_Y}$, feature maps
$\phi(x) := k_X(\cdot, x) \in \mathcal H_X$ and
$\psi(y) := k_Y(\cdot, y) \in \mathcal H_Y$, and mean embeddings
$\mu_X := \E[\phi(X)]$, $\mu_Y := \E[\psi(Y)]$, the HSIC
\citep{gretton2008} is
\begin{equation}
\HSIC(P_{XY}) \;=\; \|C_{XY}\|^2_{\mathrm{HS}},
\qquad C_{XY} \;:=\; \E\!\left[(\phi(X) - \mu_X)\otimes(\psi(Y) - \mu_Y)\right],
\label{eq:hsic-pop}
\end{equation}
the squared Hilbert--Schmidt norm of the cross-covariance operator
$C_{XY}\in\mathcal H_X\otimes\mathcal H_Y$, with $\otimes$ denoting
the tensor product.  $\HSIC(P_{XY}) = 0$ if and only if $X\indep Y$.
The biased empirical estimator is~\eqref{eq:hsic-emp}; under $H_0$ it
is a degenerate $V$-statistic.  The $\dHSIC$ generalisation
\citep{pfister2018} for $d \geq 2$ random variables $X^1, \ldots, X^d$
with bounded characteristic kernels $k_1, \ldots, k_d$, RKHSs
$\mathcal H_1, \ldots, \mathcal H_d$ (writing $\mathcal H_k$ for the
RKHS of $k_k$), feature maps $\phi_k(x) := k_k(\cdot, x)$, and joint
mean embedding
$\Pi(P) := \E_{X\sim P}[\phi_1(X^1)\otimes\cdots\otimes\phi_d(X^d)]
\in \bigotimes_{k=1}^{d}\mathcal H_k$, is the squared distance
\begin{equation}
\dHSIC(P_{(X^1,\ldots,X^d)}) \;=\; \big\|\Pi(P_{X^1}\otimes\cdots\otimes P_{X^d})
   - \Pi(P_{(X^1,\ldots,X^d)})\big\|^2_{\bigotimes_{k=1}^{d} \mathcal H_k},
\label{eq:dhsic-pop}
\end{equation}
which equals zero iff $X^1, \ldots, X^d$ are jointly independent.
Its empirical estimator is a $2d$-th order $V$-statistic.

For samples $\{X_i\}_{i=1}^n\sim P$ and $\{Y_i\}_{i=1}^n\sim Q$, set
$Z_i := (X_i, Y_i)$ and write $\mathcal F_i := (Z_1, \ldots, Z_i)$ for
the observations seen up to time $i$.  $\mMMD$ \citep{chatterjee2025}
is the studentised statistic $\eta_n := T_n / \sigma_n$, where $T_n$
is as in~\eqref{eq:mmmd} and
$\sigma_n^2 := \tfrac{1}{n^2}\sum_{i=2}^n\bigl(\tfrac{1}{i}\sum_{j<i}
\Phi^{\mMMD}(Z_i,Z_j)\bigr)^2$.  Under $H_0:P=Q$ and standard moment
conditions, $\eta_n \xrightarrow{d}\mathcal N(0,1)$ (convergence in
distribution) regardless of $P$
\citep[Theorems 3.1 \& 3.3]{chatterjee2025}; under any fixed
$P\neq Q$ the centred and rescaled $T_n$ is again asymptotically
Gaussian \citep[Theorem 5.1]{chatterjee2025}, so the test is
consistent, at a per-test cost of $O(n^2)$, identical to the
standard quadratic-time MMD.

\section{The martingale HSIC}
\label{sec:mhsic}

\subsection{Definition of the statistic}
\label{sec:mhsic-construction}

The starting point is that the HSIC $V$-statistic in~\eqref{eq:hsic-emp}
can be read entrywise as a sum of products of two empirically centred
kernel evaluations,
$\widehat\HSIC_n = \tfrac{1}{n^2}\sum_{i,j}(HK_XH)_{ij}\,(HK_YH)_{ij}$.
At the population level, this is a $V$-statistic in the bivariate kernel
\begin{equation}
\Phi^{\HSIC}(z_i, z_j) \;:=\; \bar k_X(X_i, X_j)\,\bar k_Y(Y_i, Y_j),
\qquad z_i = (X_i, Y_i),
\label{eq:Phi-hsic}
\end{equation}
with $\bar k_X, \bar k_Y$ the population-centred kernels
of~\eqref{eq:centred-kernel}.  The next lemma is the central building block
of our definition of the martingale test statistic for independence.

\begin{lemma}[Conditionally centred independence kernel]
\label{lem:Phi-cond}
Under $H_0 : X \indep Y$, for every $j \leq i-1$,
\begin{equation}
\E\!\left[\Phi^{\HSIC}(Z_i, Z_j)\,\big|\,\mathcal F_{i-1}\right] \;=\; 0.
\label{eq:phi-cond-zero}
\end{equation}
\end{lemma}
\begin{proof}
For $j \leq i-1$, $Z_j$ is part of $\mathcal F_{i-1}$ while $Z_i$ is
independent of $\mathcal F_{i-1}$.  Therefore
$\E[\Phi^{\HSIC}(Z_i, Z_j)\,|\,\mathcal F_{i-1}] =
\E_{Z_i}[\bar k_X(X_i, X_j)\,\bar k_Y(Y_i, Y_j)]$.  Under $H_0$, $Z_i$
has the product law $P_X\otimes P_Y$ and the inner expectation
factorises as $\E_{X_i}[\bar k_X(X_i, X_j)] \cdot \E_{Y_i}[\bar k_Y(Y_i,
Y_j)]$.  
By the definition of the centred kernel, each factor is zero
($\E_{X_i}[k_X(X_i, X_j)] = \mu_X(X_j)$ exactly cancels the $-\mu_X(X_j)$
term, and the remaining two centring constants cancel against each
other), so the product vanishes.
\end{proof}

In analogy with~\eqref{eq:mmmd}, define
\begin{equation}
\widehat T_n^{\mHSIC} \;=\; \frac{1}{n}\sum_{i=2}^{n}\frac{1}{i}\sum_{j=1}^{i-1}
(HK_XH)_{ij}\,(HK_YH)_{ij},
\label{eq:mhsic-T-emp}
\end{equation}
i.e., the lower-triangular sum of the Hadamard product of the two
empirically centred Gram matrices.  
Its self-normalised counterpart is
\begin{equation}
\eta_n^{\mHSIC} \;:=\; \frac{\widehat T_n^{\mHSIC}}{\widehat\sigma_n^{\mHSIC}},
\quad
\bigl(\widehat\sigma_n^{\mHSIC}\bigr)^2 \;=\; \frac{1}{n^2}\sum_{i=2}^{n}
\Bigl(\frac{1}{i}\sum_{j=1}^{i-1}(HK_XH)_{ij}\,(HK_YH)_{ij}\Bigr)^2.
\label{eq:mhsic-eta}
\end{equation}
The empirical centring breaks the strict martingale-difference property
of $\Phi^{\HSIC}$ because $HKH$ is computed from \emph{all} of the
sample, not the past only.  
The replacement is asymptotically negligible because
the gap $(HK_XH)_{ij} - \bar k_X(X_i, X_j) = O_P(1/\sqrt n)$ uniformly
in $(i, j)$ by the $\sqrt n$-consistency of the empirical mean
embedding for bounded characteristic kernels
\citep[Theorem 27]{tolstikhin2017}, and the resulting bias in
$\widehat T_n^{\mHSIC}$ is $O_P(1/n)$, contributing $o_P(1)$ to the
studentised statistic.  
We make this precise in Theorem~\ref{thm:mhsic-null}.

\subsection{Null distribution and consistency}
\label{sec:mhsic-null}

\begin{assumption}[Bounded fourth moments]
\label{ass:mhsic-moments}
The kernels $k_X, k_Y$ are bounded characteristic kernels and the
centred-kernel product satisfies $\E[\Phi^{\HSIC}(Z_1, Z_2)^2]\in
(0,\infty)$ and $\E[\Phi^{\HSIC}(Z_1, Z_2)^4] < \infty$ for i.i.d. $Z_1, Z_2$
This is the analogue of the moment condition of
\citet[Theorem 3.1]{chatterjee2025} and is satisfied for any bounded
kernel under $\E[\bar k_X(X_1,X_2)^4] + \E[\bar k_Y(Y_1,Y_2)^4] <
\infty$ (e.g.\ Gaussian and Laplace kernels).
\end{assumption}

\begin{theorem}[Null distribution of $\mHSIC$]
\label{thm:mhsic-null}
Under Assumption~\ref{ass:mhsic-moments} and $H_0:X\indep Y$,
\begin{equation}
\eta_n^{\mHSIC} \;\xrightarrow{d}\; \mathcal N(0, 1) \qquad \text{as }
n\to\infty,
\end{equation}
regardless of the joint distribution $P_{XY}$.
\end{theorem}
\begin{proof}[Proof sketch]
\emph{(i) Population centring.}  Replace $(HK_XH)_{ij}$ by $\bar k_X(X_i,
X_j)$ and likewise for $Y$, and call the resulting statistic
$\eta_n^{\mHSIC,\star}$.  By Lemma~\ref{lem:Phi-cond} the
inner sums $\tfrac{1}{i}\sum_{j<i}\Phi^{\HSIC}(Z_i, Z_j)$ form a
martingale-difference sequence with respect to $\mathcal F_{i-1}$; by
Assumption~\ref{ass:mhsic-moments}, $\Phi^{\HSIC}$ is symmetric,
bounded and has finite fourth moment.  The Berry--Esseen bounds for
self-normalised martingales of \citet{fan2018berry}, reproduced in
the proof of \citet[Theorem 3.3]{chatterjee2025} for the two-sample
$\mMMD$, yield $\eta_n^{\mHSIC,\star}\xrightarrow{d}\mathcal N(0,1)$.

\emph{(ii) Empirical centring is asymptotically negligible.}  Define
$\Delta_{ij}^X := (HK_XH)_{ij} - \bar k_X(X_i, X_j)$ and likewise
$\Delta_{ij}^Y$, and write $\hat\mu_X^{(n)} := \tfrac{1}{n}\sum_a
\phi(X_a) \in \mathcal H_X$ for the empirical mean embedding of
$\{X_i\}_{i=1}^n$ (and analogously
$\hat\mu_Y^{(n)} \in \mathcal H_Y$).  Boundedness of $k_X$ and the
$\sqrt n$-consistency of $\hat\mu_X^{(n)}$
\citep[Theorem 27]{tolstikhin2017} imply $|\Delta_{ij}^X| = O_P(1/\sqrt
n)$ uniformly in $(i,j)$, and likewise for $Y$.  The leading cross term
$\bar k_X\,\Delta^Y$ in the expansion of $(HK_XH)(HK_YH) - \bar k_X\bar
k_Y$ has conditional mean zero given $\mathcal F_{i-1}$ under $H_0$
(by an argument identical to Lemma~\ref{lem:Phi-cond} applied to the
factor $\bar k_X$), so its contribution to $\widehat T_n^{\mHSIC} -
\widehat T_n^{\mHSIC,\star}$ is $O_P(1/n)$, contributing
$o_P(1)$ to the studentised statistic.

\emph{(iii) Slutsky.}  The denominator $\widehat\sigma_n^{\mHSIC}$
converges in probability to $\widehat\sigma_n^{\mHSIC,\star}$ by the
same expansion (squared empirical-centring corrections are even
smaller).  Slutsky's theorem closes the argument.  Full proof in
Appendix~\ref{app:mhsic-proof}.
\end{proof}

\subsection{Algorithm and complexity}
\label{sec:mhsic-algorithm}

\begin{algorithm}[t]
\caption{Permutation-free independence test ($\mHSIC$)}
\label{alg:mhsic-test}
\begin{algorithmic}[1]
\REQUIRE Sample $\{(X_i, Y_i)\}_{i=1}^{n}$, level $\alpha$, kernels $k_X, k_Y$.
\STATE Compute $K_X, K_Y \in \R^{n\times n}$.
\STATE Centre: $\bar K_X \leftarrow H K_X H$, $\bar K_Y \leftarrow H K_Y H$.
\STATE $u_i \leftarrow \tfrac{1}{i}\sum_{j=1}^{i-1}(\bar K_X)_{ij}(\bar K_Y)_{ij}$ for $i = 2,\ldots,n$.
\STATE $\widehat T_n \leftarrow \tfrac{1}{n}\sum_{i=2}^{n} u_i$;\quad
       $\widehat\sigma_n^2 \leftarrow \tfrac{1}{n^2}\sum_{i=2}^{n} u_i^2$.
\STATE $\eta_n \leftarrow \widehat T_n / \widehat\sigma_n$.
\IF{$\eta_n > z_{1-\alpha}$}\STATE Reject $H_0$.\ELSE \STATE Fail to reject $H_0$.\ENDIF\quad\COMMENT{$z_{1-\alpha}$: the $(1-\alpha)$-quantile of $\mathcal N(0,1)$.}
\end{algorithmic}
\end{algorithm}

The full procedure is summarised in Algorithm~\ref{alg:mhsic-test}.
The dominant cost is the formation of the two centred Gram matrices at
$O(n^2)$ each, the lower-triangular
reduction at $O(n^2)$, and the studentisation at $O(n)$, for a total
per-test cost of $O(n^2)$, which is identical to the biased HSIC
$V$-statistic, but \emph{without} the $O(B)$ permutation
resampling step.

\section{The martingale d-variable HSIC}
\label{sec:mdhsic}

\subsection{Empirical-centring bias of the direct-product analogue}
\label{sec:mdhsic-why-split}

The natural $d$-variable analogue of $\mHSIC$ would replace the
centred-kernel pair $\bar k_X\,\bar k_Y$ by the $d$-fold centred-kernel
product
\begin{equation}
\Phi^{\dHSIC,\,\star}(z_i, z_j) \;:=\; \prod_{k=1}^{d}\bar k_k(X_i^k, X_j^k),
\qquad z_i = (X_i^1, \ldots, X_i^d),
\label{eq:Phi-d-star}
\end{equation}
with $\bar k_k$ the \emph{population} centred kernel of the $k$-th
marginal.  Under joint independence, this bivariate kernel is
conditionally centred given the past for every $j \leq i-1$:
\begin{equation}
\E\!\left[\Phi^{\dHSIC,\,\star}(Z_i, Z_j)\,\big|\,\mathcal F_{i-1}\right]
  \;=\; \prod_k \E_{X_i^k}[\bar k_k(X_i^k, X_j^k)] \;=\; 0,
\label{eq:Phi-d-star-cond}
\end{equation}
because $(X_i^1, \ldots, X_i^d) \sim \prod_k P_{X^k}$ under $H_0$
factorises and each factor integrates to zero by definition of
$\bar k_k$.  Substituting this directly into the $\mHSIC$ scheme of
Algorithm~\ref{alg:mhsic-test} yields a statistic with the right
population limit but \emph{incorrect finite-sample calibration at
$d\geq 3$}: the population kernel $\bar k_k$ is unknown and has to be
replaced by an empirically centred one, $(HK_kH)_{ij}$, and the
resulting per-variable bias $\Delta_{ij}^k =
(HK_kH)_{ij} - \bar k_k(X_i^k, X_j^k) = O_P(1/\sqrt n)$ multiplies
across the $d$ variables.  The aggregate bias is $O_P(d/\sqrt n)$, which
dominates the $\sqrt n$ studentised scale once $d \gtrsim \sqrt n$,
breaking the standard-normal calibration of the direct-product
statistic in that regime.  This formal obstruction motivates the
split-martingale construction below.

\subsection{Definition of the split-martingale statistic}
\label{sec:mdhsic-construction}

Our proposed solution is to estimate the centring from a disjoint subsample, in
the spirit of the cross-statistic $\xdHSIC$ of \citet{liu2025} (and the
earlier two-sample $\xHSIC$ of \citet{shekhar2023}), but to retain the
lower-triangular martingale pairing of $\mHSIC$ on the remaining
subsample.  Concretely, fix
\begin{equation}
m \;:=\; \lfloor n/2 \rfloor, \qquad
\mathcal S_1 := \{Z_1, \ldots, Z_m\}, \qquad
\mathcal S_2 := \{Z_{m+1}, \ldots, Z_n\},
\label{eq:split}
\end{equation}
and for each kernel $k$ let
\begin{equation}
\hat\mu_k^{\mathcal S_1} \;:=\; \tfrac{1}{m}\sum_{l=1}^{m} \phi_k(X_l^k), \qquad
\bar k_k^{\mathcal S_1}(x, x') \;:=\; \langle\phi_k(x) - \hat\mu_k^{\mathcal S_1},\,
\phi_k(x') - \hat\mu_k^{\mathcal S_1}\rangle_{\mathcal H_k}
\label{eq:split-centred-kernel}
\end{equation}
be the empirical mean embedding and the associated $\mathcal
S_1$-centred kernel.  The bivariate kernel that we substitute into the
martingale construction is the centred-kernel product evaluated with
this $\mathcal S_1$-centring:
\begin{equation}
\Phi^{\dHSIC}(z_i, z_j) \;:=\; \prod_{k=1}^{d} \bar k_k^{\mathcal S_1}(X_i^k, X_j^k),
\qquad i, j \in \{m+1, \ldots, n\}.
\label{eq:Phi-d-split}
\end{equation}
The studentised statistic is the $\mHSIC$ lower-triangular
self-normalised martingale restricted to $\mathcal S_2$:
\begin{equation}
\widehat T_n^{\mdHSIC} \;:=\; \frac{1}{n - m}\sum_{i=m+2}^{n}\frac{1}{i-m}\sum_{j=m+1}^{i-1}
\Phi^{\dHSIC}(Z_i, Z_j),
\label{eq:mdhsic-T}
\end{equation}
\begin{equation}
\bigl(\widehat\sigma_n^{\mdHSIC}\bigr)^2 \;:=\; \frac{1}{(n-m)^2}\sum_{i=m+2}^{n}
\Bigl(\frac{1}{i-m}\sum_{j=m+1}^{i-1}\Phi^{\dHSIC}(Z_i, Z_j)\Bigr)^2,
\label{eq:mdhsic-sigma}
\end{equation}
\begin{equation}
\eta_n^{\mdHSIC} \;:=\; \widehat T_n^{\mdHSIC}\,/\,\widehat\sigma_n^{\mdHSIC}.
\label{eq:mdhsic-eta}
\end{equation}
The asymptotic $\mathcal N(0, 1)$ calibration at every $d \geq 2$
rests on the following lemma.

\begin{lemma}[Conditional mean of the split-centred product kernel]
\label{lem:Phi-d-split-cond}
Assume that, when the kernels $k_1, \ldots, k_d$ depend on tunable
hyperparameters (e.g.\ the bandwidth $\sigma_k$ of a Gaussian or
Laplace kernel $k_k$), those hyperparameters
$\{\sigma_k\}_{k=1}^{d}$ are deterministic functions of $\mathcal S_1$
alone, so that $\hat\mu_k^{\mathcal S_1}$ and
$\bar k_k^{\mathcal S_1}$ are determined by $\mathcal F_m =
(Z_1, \ldots, Z_m)$.  Under joint independence, for every
$m + 1 \leq j < i$,
\begin{equation}
\E\!\left[\Phi^{\dHSIC}(Z_i, Z_j)\,\big|\,\mathcal F_{i-1}\right]
   \;=\; \prod_{k=1}^{d}\,\bigl\langle \mu_k - \hat\mu_k^{\mathcal S_1},\,
   \phi_k(X_j^k) - \hat\mu_k^{\mathcal S_1}\bigr\rangle_{\mathcal H_k}
   \;=\; O_P(m^{-d/2}).
\label{eq:Phi-d-split-cond}
\end{equation}
\end{lemma}
\begin{proof}
$Z_j$ and $\hat\mu_k^{\mathcal S_1}$ are determined by $\mathcal F_{i-1}$,
and under joint independence $(X_i^1, \ldots, X_i^d) \sim \prod_k
P_{X^k}$ factorises.  Evaluating one factor at a time:
\begin{equation*}
\E_{X_i^k}\bigl[\bar k_k^{\mathcal S_1}(X_i^k, X_j^k)\bigr]
  = \E_{X_i^k}\bigl[\langle\phi_k(X_i^k) - \hat\mu_k^{\mathcal S_1},
  \phi_k(X_j^k) - \hat\mu_k^{\mathcal S_1}\rangle\bigr]
  = \langle \mu_k - \hat\mu_k^{\mathcal S_1},
  \phi_k(X_j^k) - \hat\mu_k^{\mathcal S_1}\rangle,
\end{equation*}
which multiplies over $k$ to the displayed product.  The
$\sqrt m$-consistency of $\hat\mu_k^{\mathcal S_1}$ for bounded
characteristic kernels \citep[Theorem 27]{tolstikhin2017} gives
$\|\mu_k - \hat\mu_k^{\mathcal S_1}\|_{\mathcal H_k} = O_P(1/\sqrt m)$
and boundedness of $k_k$ gives
$\|\phi_k(X_j^k) - \hat\mu_k^{\mathcal S_1}\|_{\mathcal H_k} = O_P(1)$,
so each factor is $O_P(1/\sqrt m)$ and the product is $O_P(m^{-d/2})$.
\end{proof}

Lemma~\ref{lem:Phi-d-split-cond} is the core of the split-martingale
construction.  The per-variable empirical-centring bias is still
$O_P(1/\sqrt m)$, but because it enters the inner expectation
\emph{multiplicatively} across the $d$ variables, the aggregate residual
is $O_P(m^{-d/2})$: exponentially small in $d$ for any $m \geq 2$
and negligible next to the studentised scale for any $d \geq 2$.  Under
the direct-product construction of
Section~\ref{sec:mdhsic-why-split} this same bias enters
\emph{additively} across the $d$ variables (via a product expansion with
a $\sqrt n$-scale leading term) and grows to $O_P(d/\sqrt n)$, which
is why the direct construction fails.

\begin{algorithm}[t]
\caption{Permutation-free joint-independence test (split-martingale $\mdHSIC$)}
\label{alg:mdhsic-test}
\begin{algorithmic}[1]
\REQUIRE Sample $\{(X_i^1,\ldots,X_i^d)\}_{i=1}^{n}$ with $n \geq 6$, level $\alpha$, kernels $k_1,\ldots,k_d$.
\STATE $m \leftarrow \lfloor n/2 \rfloor$;\quad $\mathcal S_1 \leftarrow \{1,\ldots,m\}$;\quad $\mathcal S_2 \leftarrow \{m+1,\ldots,n\}$;\quad $n_2 \leftarrow n - m$.
\FOR{$k = 1,\ldots,d$}
  \STATE Compute the $n_2\!\times\! n_2$ sub-Gram $K_k^{(22)}$ on $\mathcal S_2$ and the $m\!\times\! n_2$ cross-Gram $K_k^{(12)}$ between $\mathcal S_1$ and $\mathcal S_2$.
  \STATE $\mu_k^{(\cdot)} \leftarrow \tfrac{1}{m}\one^\top K_k^{(12)}$ ($n_2$-vector);\quad $\nu_k \leftarrow \tfrac{1}{m^2}\one^\top K_k^{(11)}\one$ (scalar).
  \STATE $\bar K_k \leftarrow K_k^{(22)} - \one\, \mu_k^{(\cdot)} - \mu_k^{(\cdot)\top}\one^\top + \nu_k\,\one\one^\top$.\quad \COMMENT{$(\bar K_k)_{ij} = \bar k_k^{\mathcal S_1}(X_{m+i}^k, X_{m+j}^k)$}
\ENDFOR
\STATE $\Pi_{ij} \leftarrow \prod_{k=1}^{d}(\bar K_k)_{ij}$ for $i, j \in \{1,\ldots,n_2\}$, $i\neq j$.
\STATE $u_i \leftarrow \tfrac{1}{i}\sum_{j=1}^{i-1}\Pi_{ij}$ for $i = 2,\ldots,n_2$.
\STATE $\widehat T_n \leftarrow \tfrac{1}{n_2}\sum_{i=2}^{n_2} u_i$;\quad
       $\widehat\sigma_n^2 \leftarrow \tfrac{1}{n_2^2}\sum_{i=2}^{n_2} u_i^2$.
\STATE $\eta_n \leftarrow \widehat T_n / \widehat\sigma_n$.
\IF{$\eta_n > z_{1-\alpha}$}\STATE Reject $H_0$.\ELSE \STATE Fail to reject $H_0$.\ENDIF
\end{algorithmic}
\end{algorithm}

\subsection{Null distribution, consistency, and complexity}
\label{sec:mdhsic-null}

\begin{assumption}[Bounded fourth moments, $d$-variable]
\label{ass:mdhsic-moments}
Each $k_k$ is a bounded characteristic kernel and the
centred-kernel product $\Phi^{\dHSIC,\,\star}$ satisfies
$\E[\Phi^{\dHSIC,\,\star}(Z_1, Z_2)^2] \in (0, \infty)$ and
$\E[\Phi^{\dHSIC,\,\star}(Z_1, Z_2)^4] < \infty$ for $Z_1, Z_2$
i.i.d.\ from $P_{(X^1, \ldots, X^d)}$.  This is satisfied for any
bounded kernel under $\sum_{k}\E[\bar k_k(X_1^k, X_2^k)^{4d}] < \infty$
(e.g.\ Gaussian and Laplace kernels).
\end{assumption}

\begin{theorem}[Null distribution of split-martingale $\mdHSIC$]
\label{thm:mdhsic-null}
Under Assumption~\ref{ass:mdhsic-moments} and
$H_0: P_{(X^1,\ldots,X^d)} = \prod_{k=1}^{d} P_{X^k}$,
\begin{equation}
\eta_n^{\mdHSIC} \;\xrightarrow{d}\; \mathcal N(0, 1) \qquad \text{as }
n \to \infty,
\end{equation}
regardless of the marginals $P_{X^1}, \ldots, P_{X^d}$ and for every
fixed $d \geq 2$.
\end{theorem}
\begin{proof}[Proof sketch]
\emph{(i) Conditional centring.}  By
Lemma~\ref{lem:Phi-d-split-cond}, the inner sum
$\xi_i = \tfrac{1}{i-m}\sum_{m < j < i} \Phi^{\dHSIC}(Z_i, Z_j)$ has
conditional expectation $O_P(m^{-d/2}) = o_P(n^{-1})$ for $d \geq 2$,
so the sequence $(\xi_i - \E[\xi_i\,|\,\mathcal F_{i-1}])_{i = m+2}^n$
is an exact martingale-difference sequence whose cumulative bias is
$o_P(1)$ after studentisation.
\emph{(ii) Self-normalised martingale CLT.}  Under
Assumption~\ref{ass:mdhsic-moments} the per-variable kernel is symmetric,
bounded, and has finite fourth moment under the marginal product law;
H\"older's inequality transports this to $\Phi^{\dHSIC}$ itself.  The
Berry--Esseen bound for self-normalised martingales of
\citet{fan2018berry}, applied in identical form in
\citet[Theorem 3.3]{chatterjee2025} and in Theorem~\ref{thm:mhsic-null}
above, yields $\eta_n^{\mdHSIC} \xrightarrow{d} \mathcal N(0, 1)$.
Full proof in Appendix~\ref{app:mdhsic-proof}.
\end{proof}

The dominant cost is the $d$ rectangular and square Gram matrices at
$O(d\,(m\,n_2 + n_2^2)) = O(d\,n^2)$ and the elementwise-product
reduction at $O(d\,n_2^2) = O(d\,n^2)$, matching the biased
$V$-statistic $\dHSIC$ of \citep[Definition 2.6]{pfister2018} up to a
factor of $4$ (for the half-sample used in the studentised sum).
In contrast, standard $\dHSIC$ inherits an $O(B\,n^2)$ permutation
calibration, in practice often with $B \geq 200$ permutations
\citep[\S 3.2]{pfister2018}.
Split-martingale $\mdHSIC$ removes the
$B$ factor entirely, exactly as the cross-statistic $\xdHSIC$
\citep{liu2025} does, while additionally using both sub-samples
symmetrically, i.e., within-half pairs that $\xdHSIC$'s rectangular
block discards still contribute to the studentised statistic.

\section{Experiments}
\label{sec:experiments}

We validate $\mHSIC$ and $\mdHSIC$ along two axes: \emph{calibration}
(empirical type-I rate at the nominal level $\alpha = 0.05$ under
$H_0$) and \emph{test power} (rejection rate under fixed alternatives).
All experiments use Gaussian kernels with the median-heuristic
bandwidth
, in a float32 PyTorch implementation on a single NVIDIA RTX PRO 6000
Blackwell Server Edition GPU (96\,GB RAM).  An anonymised code bundle in
the supplementary material reproduces all results described henceforth.

\subsection{Data-generating process in the bivariate setting}
Let $X\sim\mathrm{Uniform}([-1,1]^{d_{\mathrm{ambient}}})$ and
$Y = a\,G(F(X)) + E$.  $F$ and $G$ are independently drawn random
mixtures of linear, cubic, and $\tanh$ activations of a random linear
map (full DGP in Appendix~\ref{app:method-comparison}), and the noise $E$
has i.i.d.\ entries drawn per trial from
$\{\mathrm{Gaussian}, \mathrm{Laplace}, \mathrm{Uniform}\}$ (unit variance,
fixed noise scale $0.25$).  $a \geq 0$ controls the
signal-to-noise ratio with $a = 0$ enforcing $X\indep Y$ exactly.
We call the per-variable input dimension
$d_{\mathrm{ambient}}$
the
\emph{ambient} dimension, to distinguish it from the integer $d$
counting jointly tested variables.

\subsection{Calibration and power in the bivariate setting}
We iterate over $d_{\mathrm{ambient}} \in \{1, 10, 50, 100, 500\}$,
$n$ ranging from $10$ to $16{,}000$ on a logarithmic grid, and
$a \in \{0.0, 0.2, 0.4, 0.6, 0.8, 1.0\}$, running $M = 1{,}000$
Monte-Carlo trials per cell and benchmarking $\mHSIC$ against
permutation-calibrated HSIC with $B = 200$ random label permutations.
When $n$ is sufficiently large relative to the ambient dimension, the
type-I rate at $a = 0$ is tight around $0.05$, by the
distribution-free calibration of Theorem~\ref{thm:mhsic-null} (at
$d_{\mathrm{ambient}} = 1$ it stays in $[0.015, 0.060]$ at every
$n$).  Power climbs monotonically in $a$ and shifts upwards with $n$.
Per-trial wall-clock times are reported in
Table~\ref{tab:mhsic-runtime}.

\begin{remark}[Finite-sample curse of dimensionality]
\label{rem:mhsic-curse}
The empirical-centring bias correction
(Appendix~\ref{app:mhsic-proof}, Step~2) leaves a residual that scales
empirically as $d_{\mathrm{ambient}}/n^{0.875}$, inflating type-I at
small $n$ (in the $d_{\mathrm{ambient}} = 500$ panel of
Figure~\ref{fig:mhsic-calibration} the rate decays from $0.36$ at
$n = 500$ to $0.07$ at $n = 16{,}000$).  As a heuristic,
$d_{\mathrm{ambient}} \lesssim n^{0.875}$ keeps type-I within an acceptable range of $\alpha$.
\end{remark}

\begin{figure}[t]
\centering
\includegraphics[width=\textwidth]{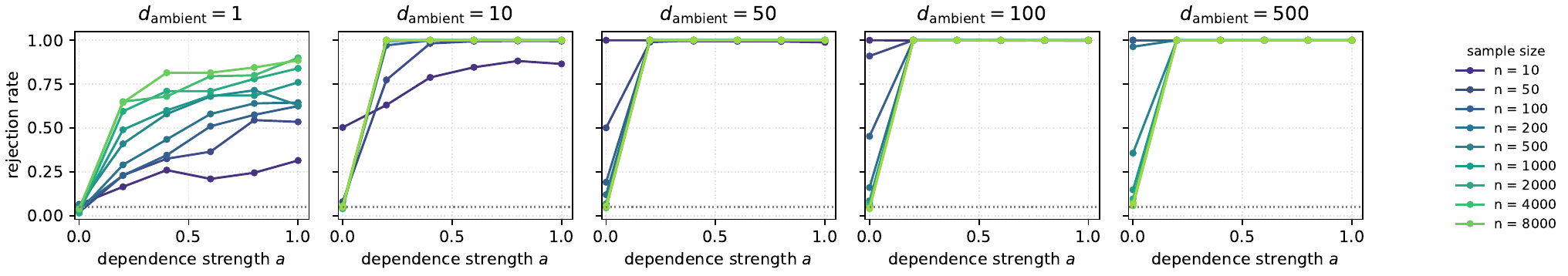}
\caption{Empirical rejection rate
at $\alpha = 0.05$ (horizontal dotted line) as a function of dependence strength
$a\in\{0.0, 0.2, 0.4, 0.6, 0.8, 1.0\}$  with $M = 1{,}000$ trials per cell of $\mHSIC$ on the random-mixture DGP.
One panel per ambient dimension $d_{\mathrm{ambient}}\in\{1, 10, 50, 100, 500\}$ and 
one curve per sample size $n \in \{10, 50, 100, 200, 500, 1{,}000, 2{,}000, 4{,}000,
8{,}000, 16{,}000\}$ (dark $\to$ light gradient).}
\label{fig:mhsic-calibration}
\end{figure}

\subsection{Joint-independence DGP and calibration in the d-variable setting}
For $d \in \{2, 3, 5, 8, 10\}$ random variables we use a
linear-Gaussian joint-independence DGP, with each
$X^k \in \R^{p}$ ($p = 5$) and, for $k \geq 2$,
$X^k = a\sum_{l<k} A_l X^l + \varepsilon^k$, where
$A_l\in\R^{p\times p}$ is drawn once per trial with i.i.d.\
$\mathcal N(0, 1/(pk))$ entries.  $a = 0$ enforces joint independence
exactly, whereas $a > 0$ produces pairwise (and therefore joint) dependence.
We run $M = 1{,}000$ trials across $n\in\{100, 500, 2{,}000, 5{,}000\}$ and
$a\in\{0.0, 0.2, 0.4, 0.6, 0.8, 1.0\}$.
Figure~\ref{fig:mdhsic-calibration} shows that the empirical type-I
rate of split-martingale $\mdHSIC$ at $a = 0$ stays in
$[0.036, 0.066]$ for every $(d, n)$, confirming the distribution-free
calibration of Theorem~\ref{thm:mdhsic-null}.  At $a > 0$, power
climbs monotonically in $a$ and grows with $n$. The detection
threshold on $a$ shifts rightward with $d$, consistent with the
$O_P(m^{-d/2})$ residual of Lemma~\ref{lem:Phi-d-split-cond}
compounding with a product-kernel signal that attenuates as $d$
grows.  We restrict $d \leq 10$, as do prior $\dHSIC$ and $\xdHSIC$
experiments \citep[\S 5]{pfister2018,liu2025}, for the same
attenuation reason.

\begin{figure}[t]
\centering
\includegraphics[width=\textwidth]{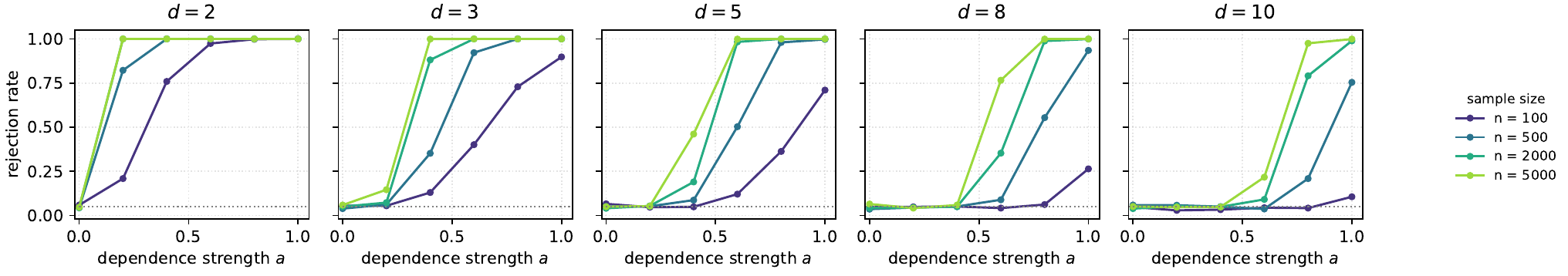}
\caption{Split-martingale $\mdHSIC$ on the linear-Gaussian
joint-independence DGP with $p = 5$, $M = 1{,}000$ repetitions per cell,
$\alpha = 0.05$.  One panel per number of variables
$d \in \{2, 3, 5, 8, 10\}$; one curve per sample size
$n \in \{100, 500, 2{,}000, 5{,}000\}$ (dark $\to$ light gradient).  The
horizontal dotted line marks the nominal level.  The value at $a = 0$
is the empirical type-I rate (range $0.036$--$0.066$ across all cells);
values at $a > 0$ estimate power.  The detection threshold on $a$
shifts rightward with $d$, consistent with the product-kernel
attenuation discussed in Section~\ref{sec:mdhsic-null}.}
\label{fig:mdhsic-calibration}
\end{figure}

\subsection{Comparison against split-based alternatives}
\label{sec:method-comparison-summary}

\begin{table}[h]
\centering
\caption{Per-test GPU wall-clock time (seconds, mean over $1{,}000$
trials) on the random-mixture DGP with $d_{\mathrm{ambient}} = 10$,
comparing $\mHSIC$ against permutation-HSIC with $B = 200$.  
}
\label{tab:mhsic-runtime}
\footnotesize
\setlength{\tabcolsep}{2.8pt}
\begin{tabular}{@{}lrrrrrr@{}}
\toprule
$n$ & $500$ & $1{,}000$ & $2{,}000$ & $4{,}000$ & $8{,}000$ & $16{,}000$ \\
\midrule
HSIC-perm & $0.028$ & $0.029$ & $0.034$ & $0.129$ & $0.650$ & $2.532$ \\
$\mHSIC$  & $0.0011$ & $0.0012$ & $0.0013$ & $0.0039$ & $0.014$ & $0.042$ \\
\midrule
Speed-up vs.\ HSIC-perm & $25\times$ & $24\times$ & $26\times$ & $33\times$ & $46\times$ & $60\times$ \\
\bottomrule
\end{tabular}
\end{table}

We additionally benchmark $\mHSIC$ against the cross-statistic
$\xHSIC$ \citep{shekhar2023} and HSIC-perm \citep{gretton2008} on the
random-mixture DGP at $d_{\mathrm{ambient}} \in \{1, 10, 500\}$ and
two sample sizes per ambient dimension, and $\mdHSIC$ against
$\xdHSIC$ \citep{liu2025} and $\dHSIC$-perm \citep{pfister2018} on
the linear-Gaussian joint-independence DGP at two $(d, n)$ cells
($M = 500$ trials, $\alpha = 0.05$, $B = 200$ permutations, all
methods sharing the same Gaussian median-heuristic bandwidths).
Inside the calibration regime, $\mHSIC$ tracks HSIC-perm at every $a > 0$ and
both dominate $\xHSIC$ at small $n$. For joint independence,
$\dHSIC$-perm has the highest power, followed by $\xdHSIC$ and then
$\mdHSIC$, with the gap reflecting the information cost of our implemented
half-sample split.
Computationally, $\mdHSIC$ runs at $O(d\,n^2)$ versus $O(B\,d\,n^2)$
for $\dHSIC$-perm, explaining the $25$ to $60\times$ speedup we observe in our experiments
(Table~\ref{tab:mhsic-runtime}).  Full results are in
Appendix~\ref{app:method-comparison}.

\section{Discussion}
\label{sec:discussion}

We presented two permutation-free kernel statistics, $\mHSIC$ and
$\mdHSIC$: studentised lower-triangular sums of conditionally
centred bivariate kernels that converge to $\mathcal N(0, 1)$ under
the relevant null regardless of the data law, at a per-test cost of
$O(n^2)$ and $O(d\,n^2)$ respectively, matching the corresponding
biased $V$-statistics without the permutation factor $B$.
Consistency against any fixed alternative follows from a calculation
parallel to \citet[Theorem 5.1]{chatterjee2025}.  $\mdHSIC$ adopts
the half-sample split of $\xdHSIC$ \citep{liu2025} for centring
only, then runs the martingale over the second half.

\bibliographystyle{plainnat}
\bibliography{references}

@inproceedings{balsubramani2016sequential,
  author    = {Balsubramani, Akshay and Ramdas, Aaditya},
  title     = {Sequential nonparametric testing with the law of the iterated logarithm},
  booktitle = {Proceedings of the Conference on Uncertainty in Artificial Intelligence (UAI)},
  year      = {2016},
}

@article{chatterjee2025,
  author  = {Chatterjee, Anirban and Ramdas, Aaditya},
  title   = {A martingale kernel two-sample test},
  journal = {arXiv preprint arXiv:2510.11853},
  year    = {2025},
}

@article{fan2018berry,
  author  = {Fan, Xiequan},
  title   = {Sharp large deviations for sums of bounded from above random variables},
  journal = {Science China Mathematics},
  volume  = {61},
  number  = {11},
  pages   = {2179--2192},
  year    = {2018},
}

@inproceedings{gretton2008,
  author    = {Gretton, Arthur and Fukumizu, Kenji and Teo, Choon Hui and Song, Le and Sch{\"o}lkopf, Bernhard and Smola, Alex},
  title     = {A kernel statistical test of independence},
  booktitle = {Advances in Neural Information Processing Systems (NeurIPS)},
  year      = {2008},
}

@article{gretton2012kernel,
  author  = {Gretton, Arthur and Borgwardt, Karsten M. and Rasch, Malte J. and Sch{\"o}lkopf, Bernhard and Smola, Alex},
  title   = {A kernel two-sample test},
  journal = {Journal of Machine Learning Research},
  volume  = {13},
  pages   = {723--773},
  year    = {2012},
}

@inproceedings{liu2025,
  author    = {Liu, Zhaolu and Peach, Robert L. and Barahona, Mauricio},
  title     = {Permutation-free high-order interaction tests},
  booktitle = {Proceedings of the International Conference on Machine Learning (ICML)},
  year      = {2025},
  note      = {arXiv:2506.05963},
}

@article{pfister2018,
  author  = {Pfister, Niklas and B{\"u}hlmann, Peter and Sch{\"o}lkopf, Bernhard and Peters, Jonas},
  title   = {Kernel-based tests for joint independence},
  journal = {Journal of the Royal Statistical Society, Series B},
  volume  = {80},
  number  = {1},
  pages   = {5--31},
  year    = {2018},
}

@article{shekhar2023,
  author  = {Shekhar, Shubhanshu and Kim, Ilmun and Ramdas, Aaditya},
  title   = {A permutation-free kernel independence test},
  journal = {Journal of the Royal Statistical Society, Series B},
  volume  = {85},
  number  = {5},
  pages   = {1701--1726},
  year    = {2023},
}

@article{shekhar2023nonparametric,
  author  = {Shekhar, Shubhanshu and Ramdas, Aaditya},
  title   = {Nonparametric two-sample testing by betting},
  journal = {IEEE Transactions on Information Theory},
  volume  = {69},
  number  = {11},
  pages   = {6890--6927},
  year    = {2023},
}

@article{tolstikhin2017,
  author  = {Tolstikhin, Ilya and Sriperumbudur, Bharath K. and Muandet, Krikamol},
  title   = {Minimax estimation of kernel mean embeddings},
  journal = {Journal of Machine Learning Research},
  volume  = {18},
  pages   = {1--47},
  year    = {2017},
}

\appendix

\section{Proof of asymptotic normality of the martingale statistic}
\label{app:mhsic-proof}

We give the full argument behind Theorem~\ref{thm:mhsic-null}.

\paragraph{Step 1: Population-centring statistic is a self-normalised
martingale-difference sum.}
With $\Phi^{\HSIC}$ defined in~\eqref{eq:Phi-hsic}, set
\begin{equation}
\widehat T_n^{\mHSIC,\star} = \frac{1}{n}\sum_{i=2}^n\frac{1}{i}\sum_{j=1}^{i-1}\Phi^{\HSIC}(Z_i, Z_j),\qquad
\bigl(\widehat\sigma_n^{\mHSIC,\star}\bigr)^2 = \frac{1}{n^2}\sum_{i=2}^n\Bigl(\frac{1}{i}\sum_{j=1}^{i-1}\Phi^{\HSIC}(Z_i, Z_j)\Bigr)^2.
\end{equation}
By Lemma~\ref{lem:Phi-cond}, under $H_0$ the inner sum
$\xi_i = \tfrac{1}{i}\sum_{j<i}\Phi^{\HSIC}(Z_i, Z_j)$ is a martingale
difference w.r.t.\ $\mathcal F_{i-1}$.  This is exactly the structure of
$\mMMD$'s $T_n$ in~\eqref{eq:mmmd} with $\Phi^{\mMMD}$ replaced by $\Phi^{\HSIC}$.
By Assumption~\ref{ass:mhsic-moments}, $\Phi^{\HSIC}$ is symmetric,
bounded, and has finite fourth moment.  The Berry--Esseen bound for
self-normalised martingales of \citet{fan2018berry} (used in identical
form in the proof of \citet[Theorem 3.1 \& 3.3]{chatterjee2025}) yields
$\widehat T_n^{\mHSIC,\star} / \widehat\sigma_n^{\mHSIC,\star}
\xrightarrow{d}\mathcal N(0, 1)$.

\paragraph{Step 2: Empirical-centring bias.}
Let $\hat\mu_X^{(n)} := \tfrac{1}{n}\sum_a \phi(X_a)\in\mathcal H_X$
be the empirical mean embedding (cf.\ Section~\ref{sec:mhsic-null})
and $e_X := \hat\mu_X^{(n)} - \mu_X$ the corresponding centring
error.  By \citet[Theorem 27]{tolstikhin2017} for bounded
characteristic kernels, $\|e_X\|_{\mathcal H_X} = O_P(1/\sqrt n)$, and
the same holds for $e_Y\in\mathcal H_Y$.  Expanding
$(HK_XH)_{ij} = \langle\phi(X_i) - \hat\mu_X^{(n)},\,\phi(X_j) -
\hat\mu_X^{(n)}\rangle_{\mathcal H_X}$ and
$\bar k_X(X_i, X_j) = \langle\phi(X_i) - \mu_X,\,
\phi(X_j) - \mu_X\rangle_{\mathcal H_X}$,
\begin{equation}
\Delta_{ij}^X := (HK_XH)_{ij} - \bar k_X(X_i, X_j) = -\langle\phi(X_i) - \mu_X, e_X\rangle
- \langle\phi(X_j) - \mu_X, e_X\rangle + \|e_X\|^2_{\mathcal H_X}.
\end{equation}
Writing $C_X := \sup_{x}\sqrt{k_X(x, x)}$ for the kernel-induced
feature-norm bound (finite by boundedness of $k_X$),
Cauchy--Schwarz gives $|\Delta_{ij}^X| \leq 2C_X\|e_X\|_{\mathcal H_X}
+ \|e_X\|^2_{\mathcal H_X} = O_P(1/\sqrt n)$ uniformly in $(i, j)$,
and likewise $|\Delta_{ij}^Y| = O_P(1/\sqrt n)$.

The product expansion is
\begin{equation}
(HK_XH)_{ij}(HK_YH)_{ij} = \bar k_X\bar k_Y + \bar k_X\Delta_{ij}^Y +
\bar k_Y\Delta_{ij}^X + \Delta_{ij}^X\Delta_{ij}^Y.
\end{equation}
The term $\Delta_{ij}^X\Delta_{ij}^Y = O_P(1/n)$ contributes $O_P(1/n)$
to $\widehat T_n^{\mHSIC} - \widehat T_n^{\mHSIC,\star}$ and
$\sqrt n \cdot O_P(1/n) = o_P(1)$ to the studentised statistic.  The
two cross terms $\bar k_X\Delta^Y$ and $\bar k_Y\Delta^X$ are individually
$O_P(1/\sqrt n)$ in absolute value but, when summed over the
lower-triangular indices and divided by $n$, telescope to
$O_P(1/n)$ contributions: applying the analogous expansion
$\Delta_{ij}^Y = -\langle\phi(Y_i) - \mu_Y, e_Y\rangle
- \langle\phi(Y_j) - \mu_Y, e_Y\rangle + \|e_Y\|^2_{\mathcal H_Y}$ (with
$e_Y := \hat\mu_Y^{(n)} - \mu_Y$) and using
$\E[\bar k_X(X_i, X_j)\,|\,\mathcal F_{i-1}, X_j] = 0$ from
Lemma~\ref{lem:Phi-cond}, the conditional expectation of each cross
term given $\mathcal F_{i-1}$ vanishes to leading order in $e_Y$, so
the cross-term contribution to the lower-triangular sum is itself
a martingale-difference sum whose magnitude is
$O_P(1/n)$.  Combining,
$\widehat T_n^{\mHSIC} - \widehat T_n^{\mHSIC,\star} = O_P(1/n)$.

\paragraph{Step 3: Variance is asymptotically equivalent.}
The same expansion gives $(\widehat\sigma_n^{\mHSIC})^2 -
(\widehat\sigma_n^{\mHSIC,\star})^2 = O_P(1/n)$, because the cross
terms entering the variance estimator are squared versions of the cross
terms in the numerator and so are at most $O_P(1/n)$.

\paragraph{Step 4: Slutsky.}
Combining Steps 1--3,
\begin{equation}
\eta_n^{\mHSIC} = \frac{\widehat T_n^{\mHSIC,\star} + O_P(1/n)}{\widehat\sigma_n^{\mHSIC,\star}\bigl(1 + o_P(1)\bigr)} \xrightarrow{d} \mathcal N(0, 1)
\end{equation}
by Slutsky's theorem.  The result follows.

\paragraph{Why the lower-triangular structure.}
The lower-triangular pairing makes the inner sum
$\xi_i = \tfrac{1}{i}\sum_{j<i}\Phi^{\HSIC}(Z_i, Z_j)$ a martingale
difference: by Lemma~\ref{lem:Phi-cond}, $\E[\xi_i\,|\,\mathcal F_{i-1}] =
\tfrac{1}{i}\sum_{j<i}\E[\Phi^{\HSIC}(Z_i, Z_j)\,|\,\mathcal F_{i-1}] =
0$.  Had we summed over all $n^2$ Hadamard entries (as in the standard
biased $V$-statistic of HSIC), the contributions sharing the index $i$
would be doubly counted across diagonals $(i, j)$ and $(j, i)$, the
diagonal entries $\Phi^{\HSIC}(Z_i, Z_i) \neq 0$ would inflate the
variance, and the standard non-degenerate $U$-statistic CLT under $H_0$
would be inapplicable due to the degeneracy of the kernel $\Phi^{\HSIC}$
under the null.  The lower-triangular restriction selects the $n(n-1)/2$
pairs that become a martingale-difference sum, and the per-row
normalisation $1/i$ restores the $O(n^2)$ asymptotic scale.

\section{Proof of asymptotic normality of the split-martingale d-variable statistic}
\label{app:mdhsic-proof}

We give the full argument behind Theorem~\ref{thm:mdhsic-null}.  Let
$m = \lfloor n/2 \rfloor$, $n_2 = n - m$,
$\mathcal S_1 = \{Z_1, \ldots, Z_m\}$,
$\mathcal S_2 = \{Z_{m+1}, \ldots, Z_n\}$, and (continuing the
shorthand of Section~\ref{sec:background}) $\mathcal F_i :=
(Z_1, \ldots, Z_i)$.  Write
\begin{equation}
\xi_i \;:=\; \frac{1}{i - m}\sum_{j=m+1}^{i-1}\Phi^{\dHSIC}(Z_i, Z_j),
\qquad i = m+2, \ldots, n,
\label{eq:xi-split}
\end{equation}
with $\Phi^{\dHSIC}$ defined in~\eqref{eq:Phi-d-split}, so
$\widehat T_n^{\mdHSIC} = (n_2)^{-1}\sum_{i=m+2}^n \xi_i$ and
$(\widehat\sigma_n^{\mdHSIC})^2 = (n_2)^{-2}\sum_{i=m+2}^n \xi_i^2$.

\paragraph{Step 1: Approximate martingale-difference decomposition.}
Write $\xi_i = \xi_i^{\star} + r_i$ where
\begin{equation}
\xi_i^{\star} \;:=\; \xi_i - \E[\xi_i\,|\,\mathcal F_{i-1}], \qquad
r_i \;:=\; \E[\xi_i\,|\,\mathcal F_{i-1}].
\end{equation}
By construction $(\xi_i^{\star})_{i = m+2}^n$ is a martingale
difference sequence adapted to $(\mathcal F_i)$; by
Lemma~\ref{lem:Phi-d-split-cond},
\begin{equation}
|r_i| \;\leq\; \prod_{k=1}^{d}\|\mu_k - \hat\mu_k^{\mathcal S_1}\|_{\mathcal H_k}
       \cdot \|\phi_k(X_j^k) - \hat\mu_k^{\mathcal S_1}\|_{\mathcal H_k}
  \;=\; O_P(m^{-d/2}),
\label{eq:r-bound}
\end{equation}
uniformly in $i > m$, where the $\|\mu_k - \hat\mu_k^{\mathcal
S_1}\|_{\mathcal H_k} = O_P(1/\sqrt m)$ bound is
\citet[Theorem 27]{tolstikhin2017} and the other factor is $O_P(1)$ by
boundedness of $k_k$.  Summing,
\begin{equation}
\widehat T_n^{\mdHSIC} \;=\; \frac{1}{n_2}\sum_{i=m+2}^n \xi_i^{\star}
   + \frac{1}{n_2}\sum_{i=m+2}^n r_i
   \;=\; \widehat T_n^{\mdHSIC,\star} + O_P(m^{-d/2}).
\end{equation}
For every fixed $d \geq 2$ and $m \geq n/2$, $O_P(m^{-d/2}) = o_P(n^{-1/2})$,
so the residual is negligible after studentisation.

\paragraph{Step 2: Self-normalised martingale CLT for the leading term.}
By Step~1, $\widehat T_n^{\mdHSIC,\star} = (n_2)^{-1}\sum_i \xi_i^{\star}$
is a self-normalised sum of martingale differences.  Under
Assumption~\ref{ass:mdhsic-moments}, $\Phi^{\dHSIC}$ is symmetric,
bounded, and has finite fourth moment conditional on $\mathcal F_m$
(boundedness and finite fourth moment of each $\bar k_k^{\mathcal
S_1}$ transport to the product by H\"older, since the $L^p$-norm of a
product of bounded zero-mean random variables is bounded by the
product of the $L^p$-norms).  The same Berry--Esseen bound for
self-normalised martingales of \citet{fan2018berry} used in the proof
of Theorem~\ref{thm:mhsic-null} applies conditionally on $\mathcal
F_m$, yielding
\begin{equation}
\widehat T_n^{\mdHSIC,\star} \,/\, \widehat\sigma_n^{\mdHSIC,\star}
  \;\xrightarrow{d}\; \mathcal N(0, 1)
  \qquad\text{as } n \to \infty,
\end{equation}
where $(\widehat\sigma_n^{\mdHSIC,\star})^2 = (n_2)^{-2}\sum_i
(\xi_i^{\star})^2$.

\paragraph{Step 3: Variance is asymptotically equivalent.}
Expanding $\xi_i^2 = (\xi_i^{\star} + r_i)^2 = (\xi_i^{\star})^2 + 2
\xi_i^{\star} r_i + r_i^2$ and summing, the cross-term is a
martingale-difference sum of magnitude $O_P(1/\sqrt{n_2} \cdot m^{-d/2})
= o_P(n^{-1})$ and the $r_i^2$ term is $O_P(m^{-d})$.  Both are
dominated by $(\widehat\sigma_n^{\mdHSIC,\star})^2$.  Hence
$(\widehat\sigma_n^{\mdHSIC})^2 = (\widehat\sigma_n^{\mdHSIC,\star})^2
(1 + o_P(1))$.

\paragraph{Step 4: Slutsky.}
Combining Steps 1--3, $\eta_n^{\mdHSIC} = (\widehat T_n^{\mdHSIC,\star}
+ o_P(n^{-1/2})) / (\widehat\sigma_n^{\mdHSIC,\star}(1 + o_P(1)))
\xrightarrow{d} \mathcal N(0, 1)$ by Slutsky's theorem.

\paragraph{Necessity of the half-sample split.}
Without the half-sample split, the centring would be the full-sample
empirical mean embedding $\hat\mu_k^{(n)} :=
\tfrac{1}{n}\sum_{a=1}^{n}\phi_k(X_a^k)\in\mathcal H_k$, which is not
determined by $\mathcal F_{i-1}$.  The product expansion then produces $d$
cross-terms of the form $\bar k_1 \cdots \bar k_{k-1}\,\Delta_{ij}^k\,
\bar k_{k+1}\cdots \bar k_d$ with
$\Delta_{ij}^k := (HK_kH)_{ij} - \bar k_k(X_i^k, X_j^k) = O_P(1/\sqrt
n)$.  When summed over the lower-triangular index set and divided by
$n$, each cross term individually contributes $O_P(1/n)$ (by the same
conditional-mean-zero argument as in
Appendix~\ref{app:mhsic-proof}), but the aggregate over all $d$ such
cross terms is $O_P(d/n)$, which after $\sqrt n$-scaling yields
$O_P(d/\sqrt n)$, dominating the studentised scale for $d \gtrsim
\sqrt n$.  The sample split replaces the full-sample centring by the
first-half empirical mean $\hat\mu_k^{\mathcal S_1}$, which is
determined by $\mathcal F_m$ and reduces the conditional-mean
residual to the product of $d$ inner products in~\eqref{eq:r-bound},
hence to $O_P(m^{-d/2})$, which is exponentially small in $d$.

\section{Method comparison against split-based alternatives}
\label{app:method-comparison}

For Section~\ref{sec:experiments}'s bivariate experiments,
$X\sim\mathrm{Uniform}([-1,1]^{d_{\mathrm{ambient}}})$ and
$Y = a\,G(F(X)) + E$ with $F(x) = \phi_F(A_F x)$,
$A_F\in\R^{d_{\mathrm{ambient}}\times d_{\mathrm{ambient}}}$ drawn
per trial with i.i.d.\ $\mathcal N(0, 1/d_{\mathrm{ambient}})$
entries, $\phi_F(z) = w_{F,1}\,z + w_{F,2}\,z^3 + w_{F,3}\,\tanh(z)$
and $(w_{F,1}, w_{F,2}, w_{F,3})\sim\mathrm{Dirichlet}(1,1,1)$; $G$
is drawn analogously with independent $A_G, w_G$.

Figures~\ref{fig:mhsic-comparison} and~\ref{fig:mdhsic-comparison}
compare $\mHSIC$ against the cross-statistic $\xHSIC$
\citep{shekhar2023} and permutation-HSIC \citep{gretton2008} on two
$(d_{\mathrm{ambient}}, n)$ cells per ambient dimension at
$d_{\mathrm{ambient}} \in \{1, 10, 500\}$ from the
Figure-\ref{fig:mhsic-calibration} grid, and split-martingale
$\mdHSIC$ against $\xdHSIC$ \citep{liu2025} and
permutation-$\dHSIC$ \citep{pfister2018} on two $(d, n)$ cells from
the Figure-\ref{fig:mdhsic-calibration} grid ($M = 500$ trials per
cell, $\alpha = 0.05$, $B = 200$ permutations, all methods sharing
the same Gaussian median-heuristic bandwidths).  In the independence
comparison (Figure~\ref{fig:mhsic-comparison}), $\mHSIC$ and
HSIC-perm track each other closely at every $a > 0$ and both
dominate $\xHSIC$ at small $n$; at $d_{\mathrm{ambient}} = 500$,
$n = 100$ the $\mHSIC$ type-I rate is inflated because the sample
falls outside the calibration regime of
Remark~\ref{rem:mhsic-curse}.  In the asymptotic cells
($n \geq 500$ at $d_{\mathrm{ambient}} \in \{1, 10\}$, $n \geq 2{,}000$
at $d_{\mathrm{ambient}} = 500$) all three methods are calibrated
and saturate at power~1 by $a = 0.2$.  In the joint-independence
comparison (Figure~\ref{fig:mdhsic-comparison}), $\dHSIC$-perm
delivers the highest power and saturates earliest in $a$, followed
by $\xdHSIC$ and then split-martingale $\mdHSIC$; all three are
correctly calibrated at $a = 0$ (with $\xdHSIC$ conservative,
rejecting below the nominal rate).  The power gap between $\mdHSIC$
and the permutation baseline grows with $d$, reflecting the
information cost of the half-sample split that $\mdHSIC$ pays to
secure distribution-free calibration.

\begin{figure}[t]
\centering
\includegraphics[width=\textwidth]{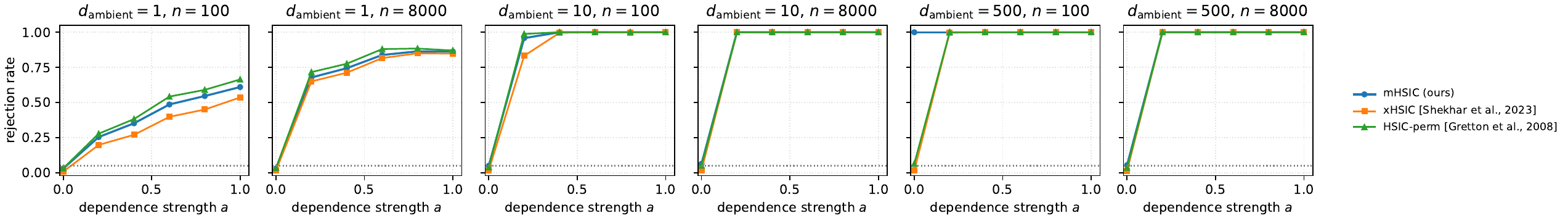}
\caption{Independence test comparison: $\mHSIC$ (blue) vs.\
$\xHSIC$ \citep{shekhar2023} (orange) vs.\ HSIC-perm
\citep{gretton2008} (green).  $M = 500$ trials, random-mixture DGP,
$\alpha = 0.05$, $B = 200$ permutations; one panel per
$(d_{\mathrm{ambient}}, n)$ cell with
$d_{\mathrm{ambient}} \in \{1, 10, 500\}$.  In the well-calibrated
regime, $\mHSIC$ and HSIC-perm have essentially identical power at
every $a > 0$; $\xHSIC$ trails at small $n$ and matches HSIC-perm by
$n = 8{,}000$.}
\label{fig:mhsic-comparison}
\end{figure}

\begin{figure}[t]
\centering
\includegraphics[width=\textwidth]{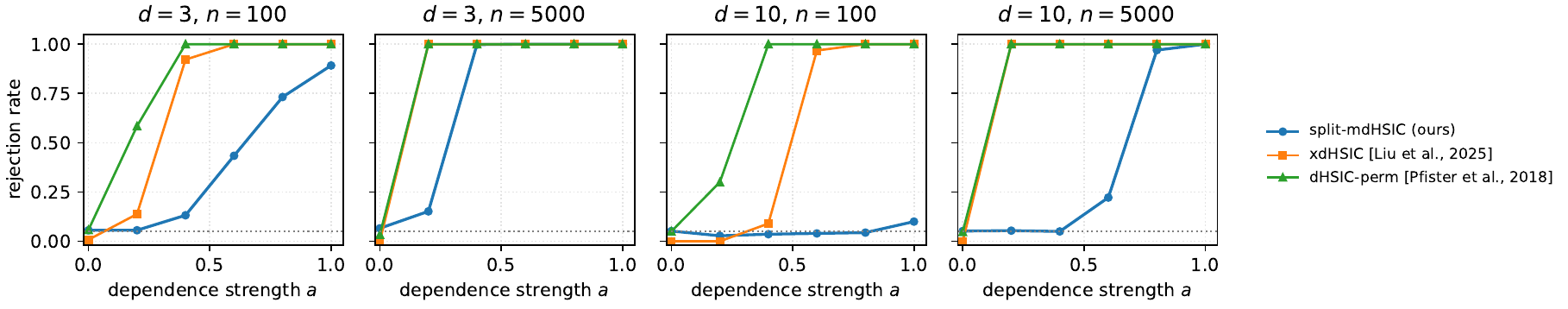}
\caption{Joint-independence test comparison: split-martingale
$\mdHSIC$ (blue) vs.\ $\xdHSIC$ \citep[orange]{liu2025} vs.\
$\dHSIC$-perm \citep[green]{pfister2018}.  $M = 500$ trials,
linear-Gaussian joint-independence DGP, $\alpha = 0.05$; one panel
per $(d, n)$ cell.  $\dHSIC$-perm has the highest power, followed by
$\xdHSIC$ and split-$\mdHSIC$; all three are calibrated at $a = 0$.
At low dependence strength $a$, $\mdHSIC$ requires more samples than
$\xdHSIC$ and $\dHSIC$-perm to reach the same test power.}
\label{fig:mdhsic-comparison}
\end{figure}

\end{document}